\documentclass[sigconf]{acmart}

\usepackage[utf8]{inputenc} 
\usepackage[T1]{fontenc}    
\usepackage{url}            
\usepackage{booktabs}       
\usepackage{amsfonts}       
\usepackage{nicefrac}       
\usepackage{microtype}      
\usepackage[page,title]{appendix}
\usepackage{comment}
\usepackage{enumitem}
\usepackage{threeparttable}

\usepackage{bm}
\usepackage{dsfont}
\usepackage{amsmath}
\usepackage{amsthm}

\usepackage{graphicx}
\usepackage{subfigure}
\usepackage{capt-of}

\usepackage{hyperref}
\usepackage{comment}

\usepackage{wrapfig}       
\usepackage{diagbox}
\usepackage{multirow}
\usepackage{makecell}

\usepackage{pifont}
\usepackage[ruled]{algorithm2e}

\def \bP{{\bf P}} 
\def \bE { {\bf E}}
\def \bR {{\bf R}}
 
\def \P{\mathbb{P}} 
\def \V {\mathbb{V}}
\def \bfE {\mathbb{E}}

 \def \cU {\mathcal{U}}
 \def \cI {\mathcal{I}}
\def \cL{\mathcal{L}}
\def \cO{\mathcal{O}}
 \def \cD{\mathcal{D}}
 \def \cH{\mathcal{H}}

\AtBeginDocument{%
  \providecommand\BibTeX{{%
    \normalfont B\kern-0.5em{\scshape i\kern-0.25em b}\kern-0.8em\TeX}}}

\copyrightyear{2022}
\acmYear{2022}
\setcopyright{acmcopyright}
\acmConference[KDD '22] {Proceedings of the 28th ACM SIGKDD Conference on Knowledge Discovery and Data Mining}{August 14--18, 2022}{Washington, DC, USA.}
\acmBooktitle{Proceedings of the 28th ACM SIGKDD Conference on Knowledge Discovery and Data Mining (KDD '22), August 14--18, 2022, Washington, DC, USA}
\acmPrice{15.00}
\acmISBN{978-1-4503-9385-0/22/08}
\acmDOI{10.1145/3534678.3539270}
\acmSubmissionID{rtfp0539}

\begin{document}

\title{A Generalized Doubly Robust Learning Framework for Debiasing Post-Click Conversion Rate Prediction}



\author{Quanyu Dai}
\affiliation{\institution{Huawei Noah's Ark Lab \country{} daiquanyu@huawei.com}}
\author{Haoxuan Li}
\affiliation{\institution{Peking University \country{} hxli@stu.pku.edu.cn}}
\author{Peng Wu$^{*}$}
\affiliation{\institution{Beijing Technology and Business University \country{} wupeng@bicmr.pku.edu.cn}}
\author{Zhenhua Dong}
\affiliation{\institution{Huawei Noah's Ark Lab \country{} dongzhenhua@huawei.com}}
\author{Xiao-Hua Zhou$^{*}$}
\affiliation{\institution{Peking University \country{} azhou@bicmr.pku.edu.cn}}

\author{Rui Zhang}
\affiliation{\institution{www.ruizhang.info \country{} rayteam@yeah.net}}

\author{Rui Zhang}
\affiliation{\institution{Huawei Hong Kong Theory Lab \country{} zhangrui191@huawei.com}}

\author{Jie Sun}
\affiliation{\institution{Huawei Hong Kong Theory Lab \country{} j.sun@huawei.com}}

\thanks{$^{*}$Peng Wu and Xiao-Hua Zhou are the corresponding authors.}

\renewcommand{\shortauthors}{Quanyu Dai et al.}



\begin{abstract}
 
Post-click conversion rate (CVR) prediction is an essential task for discovering user interests and increasing platform revenues in a range of industrial applications. One of the most challenging problems of this task is the existence of severe selection bias caused by the inherent self-selection behavior of users and the item selection process of systems. Currently, doubly robust (DR) learning approaches achieve the state-of-the-art performance for debiasing CVR prediction. However, in this paper, by theoretically analyzing the bias, variance and generalization bounds of DR methods, we find that existing DR approaches may have poor generalization caused by inaccurate estimation of propensity scores and imputation errors, which often occur in practice. Motivated by such analysis, we propose a generalized learning framework that not only unifies existing DR methods, but also provides a valuable opportunity to develop a series of new debiasing techniques to accommodate different application scenarios. Based on the framework, we propose two new DR methods, namely DR-BIAS and DR-MSE. DR-BIAS directly controls the bias of DR loss, while DR-MSE balances the bias and variance flexibly, which achieves better generalization performance. In addition, we propose a novel tri-level joint learning optimization method for DR-MSE in CVR prediction, and an efficient training algorithm correspondingly. We conduct extensive experiments on both real-world and semi-synthetic datasets, which validate the effectiveness of our proposed methods.
\end{abstract}




\maketitle
\sloppy

\section{Introduction}

The post-click conversion rate (CVR) prediction has gained much attention in modern recommender systems \cite{RecSys_Saito20, ESMM18, Wu-etal2018, MRDR_DL, wu2022causal}, as post-click conversion feedback contains strong signals of user preference and directly contributes to the gross merchandise volume (GMV). In many industrial applications, 
CVR prediction is commonly regarded as the central task for discovering user interests and increasing platform revenues.
For a user-item pair, CVR represents the probability of the user consuming the item after he/she clicks it.
Essentially, the task of CVR prediction is a {\bf counterfactual} problem. This is because what we want to know during inference is intrinsically the conversion rates of all user-item pairs under the assumption that all items are clicked by all users, which is a hypothetical situation that contradicts reality. 

Most of the literature treats CVR prediction as a missing data problem in which the conversion labels are observed in clicked events and missing in unclicked events. A conventional and natural strategy is to train the CVR model only based on clicked events and then predict CVR for all the events~\cite{KDD_LuPWPWY17,IJCAI_SuZDZYWBXHY20}. However, this estimator is biased and often obtains a sub-optimal result due to the existence of severe selection bias~\cite{ESMM18,MRDR_DL}. In addition, the data sparsity issue, namely, the sample size of clicked events being much smaller than that of unclicked events, will amplify the difference between these two types of events and thus aggravate the selection bias issue.

Several approaches have been proposed to derive unbiased estimators of CVR by dealing with selection bias. 
Error imputation~\cite{chang2010training} and inverse propensity score (IPS) weighting  \cite{Multi_IPW,Schnabel-etal2016} are two main strategies for debiasing CVR prediction. 
In addition, Doubly robust (DR) estimators can be constructed by combining EIB and IPS approaches \cite{Multi_IPW,RecSys_Saito20,Wu-etal2018}. 
A DR estimator enjoys the property of double robustness, which guarantees the unbiased estimation of CVR if either the imputed errors or propensity scores are accurate. 
Compared with EIB and IPS methods, the DR method has a better performance in general~\cite{Wang-Zhang-Sun-Qi2019}. 


There are still some concerns for DR methods, even though they usually compare favorably with EIB and IPS estimators. 
Theoretical analysis of DR estimators in Section~\ref{dr_concerns} shows that the bias, variance and generalization bounds all depend on the error deviation of the imputation model weighted by the inverse of propensity score. This is a worrying result, because the inverse of propensity score tends to be large in unclicked events and error deviations of the imputation model are most likely to be inaccurate in unclicked events due to the selection bias and data sparsity. 
It indicates that the bias, variance and generalization bounds may still be large 
under inaccurate imputed errors in unclicked events.  
Recently, several approaches, mainly including doubly robust joint learning (DR-JL) 
\cite{Wang-Zhang-Sun-Qi2019} and more robust doubly robust (MRDR) \cite{MRDR_DL}, have been designed to alleviate this problem. 
MRDR aims to reduce the variance of DR loss to enhance model robustness, but it may still have poor generalization when the bias is large. 
DR-JL attempts to reduce the error deviation of the imputation model
in order to obtain a more accurate estimator of CVR, but this method does not
 control the bias and variance directly. 
 Therefore, it would be helpful if we could find a more effective way to control the bias and variance directly.    

In this paper, we reveal the counterfactual issues behind the CVR prediction task and give a formal and strict causal definition of CVR. 
Then, by analyzing the bias, variance and generalization bound of the DR estimator, we derive a novel generalized learning framework that can accommodate a wide range of CVR estimators through specifying different metrics of loss functions. 
This framework unifies various existing doubly robust methods for debiasing CVR prediction, such as DR-JL and MRDR.
Most importantly, it provides key insights for designing new estimators to accommodate different application scenarios in CVR prediction. 
Based on this framework, from a perspective of bias-variance trade-off, we propose two new doubly robust estimators, called {\bf DR-BIAS} and {\bf DR-MSE}, which are designed to more flexibly control the bias and mean squared error (MSE) of DR loss function, respectively. 
DR-MSE achieves better generalization performance based on our analysis compared with existing DR based methods. In addition, we propose a novel tri-level joint learning optimization method for flexible DR-MSE in CVR prediction, and an efficient training algorithm correspondingly.
Extensive experiments are carried out to validate the advantages of the proposed methods compared with state-of-the-art techniques. DR-MSE outperforms them up to 3.22\% in DCG@2 in our experiments.

The main contributions of this paper can be summarized as follows:
(1) We propose a generalized framework of doubly robust learning, which not only unifies the existing DR methods, but also provides key insights for designing new estimators with different requirements to accommodate different application scenarios.
(2) Based on the proposed framework, we design two new doubly robust methods, called DR-BIAS and DR-MSE, which can better control the bias and mean squared error, compared with existing methods.
(3) For the bias-variance tradeoff parameter of DR-MSE, we propose a tri-level DR-MSE joint learning optimization for the CVR prediction task, and an efficient training algorithm correspondingly.
(4) Experimental results on both \textbf{real-world} and \textbf{semi-synthetic} datasets show that the two proposed doubly robust methods outperform the state-of-the-art methods significantly. Especially, both datasets with missing-at-random ratings and large industrial dataset are used for comprehensive evaluation. 

\section{Preliminaries}
In this section, we uncover the counterfactual feature of CVR prediction task 
within the potential outcome framework~\cite{Rubin1974,Imbens-Rubin2015}, 
and discuss some existing approaches for CVR prediction. 

\subsection{Causal Problem Definition}
Notation is described as follows.  
Let $\cU = \{1, 2, ..., m\}$ and  $\cI = \{1, 2, ..., n\}$ be the sets of $m$ users and $n$ items, respectively, and $\cD = \cU \times \cI$ be the set of all user-item pairs. Let $x_{u,i}$ be the feature vector of user $u$ and item $i$, and   $r_{u,i} \in \{0, 1\}$ be the indicator of the observed conversion label.
Let $o_{u,i}$ be the indicator of a click event, i.e., $o_{u,i} = 1$ if user $u$ clicks item $i$, $o_{u,i} = 0$ otherwise. Then, $\cO = \{ (u,i) \mid (u,i)\in \cD, o_{u,i} =  1 \}$ denotes all the clicked events.   

For any user-item pair $(u,i)$, we are interested in predicting the CVR {\bf if} user $u$  had clicked item $i$. Notice in particular that the word ``{\bf if}'' is {\bf counterfactual}.
Specifically, in the real world, each user clicks only some items and many items have never been clicked by some users, but what we want to know is the conversion rates of all the user-item pairs when each user clicks all items, which is a hypothetical situation 
 that contradicts reality.
 
Potential outcome is a basic tool to delineate counterfactual quantity in causal inference \cite{Imbens-Rubin2015}.
Through it, the task of predicting CVR can be 
defined formally.  Concretely,  we treat  $o_{u,i}$ as a treatment (or an intervention) and define the potential conversion label   
   $r_{u,i}(1)$, which represents the conversion label of a user $u$ on an item $i$ if the item is clicked by the user. Correspondingly, $r_{u,i}(0)$ is defined as the conversion label if the user $u$ did not click the item $i$.
Then the CVR can be fundamentally defined as
    {
  \begin{equation}\label{cvr}     
  \P ( r_{u,i}(1) = 1 \mid X_{u,i}= x_{u,i}  ),      
  \end{equation}	
  }
which is a causal definition and it is 
coherent and consistent with the practical implications of CVR in recommender systems.  In comparison,  the conventional definition of CVR (see \cite{ESMM18} ), defined by  
		$ \P ( r_{u,i} = 1 \mid X_{u,i}= x_{u,i},  o_{u,i}  = 1 )$, 
is based on association (or correlation) and  lost the meaning of ``counterfactual''.  
For estimating CVR in Equation (\ref{cvr}), a fundamental challenge is that only one of the potential outcome $(r_{u,i}(1), r_{u,i}(0))$ is observable. 
By consistency assumption, $r_{u,i}(1)$ is observed  when $o_{u,i} = 1$, missing otherwise. 
Therefore, the goal of estimating CVR can be recast into a missing data problem.  

For ease of presentation, we denote $\bR \in \{0,  1\}^{m\times n}$ as the full potential conversion label matrix with each element being $r_{u,i}(1)$, and let $\bR^{o} = \{ r_{u,i}(1) \mid (u,i) \in \cO \} = \{ r_{u,i} \mid (u,i) \in \cO \}$ be the set consisting of potential conversion labels $r_{u,i}(1)$ in clicked events.  
Let $\hat \bR \in [0, 1]^{m\times n}$ be the predicted conversion rate matrix, where each entry $\hat r_{u,i}(1) \in [0, 1]$ denotes the predicted conversion rate obtained by a model $f_{\phi}(x_{u,i})$ with parameters $\phi$. If the full potential conversion label matrix $\bR$ was observed,  the ideal loss function is 
        {
		\begin{equation}
				\cL_{ideal}(\hat \bR, \bR)  
				=  \frac{1}{|\cD|}  \sum_{(u,i) \in \cD} e_{u,i},   
		\end{equation}    
		}
where $e_{u,i}$ is the prediction error. In this paper, we employ the cross entropy loss  
 			$e_{u,i}  
 			= - r_{u,i}(1) \log \{\hat r_{u,i}(1) \} - \{1- r_{u,i}(1)\} \log \{1 - \hat r_{u,i}(1)\}$. 
	{$\cL_{ideal}(\hat \bR, \bR)$} can be regarded as a benchmark of unbiased loss function theoretically, even though it is infeasible  due to the 
	inaccessibility of $\bR$ practically.

\subsection{Existing Methods}\label{existing_methods}
A direct method is to use the following loss function 
{
	 $ \cL_{naive}(\hat \bR, \bR^{o})  = 
		 |\cO|^{-1} \sum_{(u,i) \in \cO } e_{u,i}$ 
}
based on the observed conversion labels $\bR^o$. {$\cL_{naive}(\hat \bR, \bR^{o})$} is not an unbiased estimate of {$\cL_{ideal}(\hat \bR, \bR)$}. Next, we will briefly review some typical and latest methods for addressing the selection bias issue.

\subsubsection{Error Imputation Based Estimator}  

The error imputation based (EIB) estimator  can be derived by introducing an  error imputation model $\hat e_{u,i} = g_{\theta}(x_{u,i})$ to fit the prediction error $e_{u,i}$. 
Given the imputed errors,  the loss function of EIB method is given as  
	$ 
			 \cL_{EIB}(\hat \bR, \bR^{o}) =   | \cD |^{-1}  \sum_{(u,i) \in \cD} [ o_{u,i} e_{u,i} + (1- o_{u,i}) \hat e_{u,i} ].
	$


\subsubsection{Inverse Propensity Score Estimator} The inverse propensity score (IPS) approach \cite{Schnabel-etal2016} aims to recover the distribution of all events by  weighting the clicked events with $1/p_{u,i}$, where $p_{u,i} = \P(o_{u,i} = 1) = \mathbb{E}[o_{u,i}]$ is the propensity score~\cite{Rosenbaum-Rubin1983}. Given the estimate of $p_{u,i}$, denoted as $\hat p_{u,i}$,  the loss function of IPS estimator  is presented as 
  $			\cL_{IPS}(\hat \bR, \bR^{o}) = |\cD|^{-1}  \sum_{(u,i) \in \cD} o_{u,i} e_{u,i} / \hat p_{u,i} $.  
				


\subsubsection{Doubly Robust Joint Learning Estimator}
Doubly robust (DR) estimator can be constructed in the augmented IPS form~\cite{Bang-Robins2005, Wang-Zhang-Sun-Qi2019} 
by combining EIB and IPS methods. Given the learned propensities $\hat \bP = \{ \hat p_{u,i} \mid (u,i) \in \cD \}$ and imputed errors $\hat \bE = \{ \hat e_{u,i} \mid (u,i) \in \cD \}$, its loss function is formulated as 
        {
		\begin{align}  \label{dr_loss}
					\cL_{DR}(\hat \bR, \bR^{o}) ={}& \frac{1}{ |\cD|} \sum_{(u,i) \in \cD}\Big [ \hat e_{u,i}  +  \frac{ o_{u,i} (e_{u,i} -  \hat e_{u,i}) }{ \hat p_{u, i} } \Big ]. 
		\end{align}}
{$\cL_{DR}(\hat \bR, \bR^{o})$} involves the conversion rate model $\hat r_{u,i}(1) = f_{\phi}(x_{u,i})$ and error imputation model $\hat e_{u,i} = g_{\theta}(x_{u,i})$. Doubly robust joint learning (DR-JL) approach \cite{Wang-Zhang-Sun-Qi2019} estimates them alternately: 
given $\hat \theta$, $\phi$ is updated by minimizing (\ref{dr_loss}); given $\hat \phi$, $\theta$ is updated by minimizing 
    {
	\begin{equation} \label{drjl_loss}
		\cL_{e}^{DR-JL}(\theta)   = 
		 \sum_{(u, i) \in \cD} \frac{ o_{u,i} (  \hat e_{u,i} - e_{u,i} )^{2}  }{    \hat p_{u,i}  }.
	\end{equation}}
\subsubsection{More Robust Doubly Robust Estimator}
 Recently, the more robust doubly robust (MRDR) method~\cite{MRDR_DL} enhances the robustness of DR-JL by optimizing the variance of the DR estimator with the imputation model. Specifically, MRDR keeps the loss of the CVR prediction model in (\ref{dr_loss}) unchanged, while replacing the loss of the imputation model in (\ref{drjl_loss}) with the following loss
    {
 	\begin{equation}\label{mrdr_loss}
		 \cL_{e}^{MRDR}(\theta) = \sum_{(u,i)\in \cD} \frac{ o_{u,i} (  \hat e_{u,i} - e_{u,i} )^{2}  }{    \hat p_{u,i}  } \cdot  \frac{ 1- \hat p_{u,i}  }{ \hat p_{u,i} }. 
	\end{equation}  
	}
This substitution can 
help reduce the variance of {$\cL_{DR}(\hat \bR, \bR^{o})$} and hence a more robust estimator might be obtained. 

\subsubsection{Bias, Variance and Generalization Bound of DR Estimator}

  Given a hypothesis space $\cH$ of CVR prediction matrix $\hat \bR$, we define the optimal
  $\hat \bR^{*}$ as 
	{$ \hat \bR^{*} = \arg \min_{ \hat \bR \in \cH } \cL_{DR}(\hat \bR, \bR^{o})$}. 
	Given  imputed errors $\hat \bE$ and learned propensities $\hat \bP$,  
 The following Lemmas \ref{lemma1} and \ref{lemma2} present the existing theoretical results of  DR estimator~\cite{schnabel2016recommendations,Wang-Zhang-Sun-Qi2019}.

\begin{lemma}[Bias and Variance]
\label{lemma1}
 The bias and variance of DR estimator are given as 
    {
	\begin{align*}
   Bias[  \cL_{DR}(\hat \bR, \bR^{o}) ] ={}&   \frac{1}{ | \cD | } \Big |  \sum_{(u,i) \in D}  (p_{u,i} - \hat p_{u,i}) \frac{ ( e_{u,i} - \hat e_{u,i})  }{ \hat p_{u,i} }  \Big |, \\
    \V_{\cO}[ \cL_{DR}(\hat \bR, \bR^{o})  ]  ={}&  \frac{1}{ |\cD|^{2} } \sum_{(u,i)\in \cD}    p_{u,i} (1- p_{u,i})  \frac{ ( \hat e_{u,i} - e_{u,i} )^{2} }{  \hat p^{2}_{u,i}  }. 
    \end{align*}}
\end{lemma}

\begin{lemma}[Generalization Bound]
\label{lemma2}
For any finite hypothesis space $\cH$ of prediction matrices,  then  with probability $1 - \eta$, 
    {
	\begin{align*} 
	  \cL_{ideal}(\hat \bR^{*}, \bR)  
	  \leq{}&  \cL_{DR}(\hat \bR^{*}, \bR^{o})  +   \underbrace{\frac{1}{ | \cD | }  \sum_{(u,i) \in \cD}   \frac{ | p_{u,i} - \hat p_{u,i} |  }{ \hat p_{u,i} }  |  e_{u,i} - \hat e_{u,i}^{*}  | }_{\text{Bias term}}  \\
	   &  +  \underbrace{ \sqrt{ \frac{ \log(2|\cH | /\eta)  }{ 2 |\cD |^{2}  }   \sum_{(u,i)\in \cD} (  \frac{  e_{u,i} - \hat e_{u,i}^{\dag}  } { \hat p_{u,i} }  )^{2}   } }_{\text{Variance term}},       
	   \end{align*}
	   } 
 where  $\hat e_{u,i}^{*}$ is the prediction error associated with  $\hat \bR^{*}$,  $\hat e_{u,i}^{\dag}$ is the prediction error corresponding to the prediction matrix $\hat \bR^{\dag} = \arg \max_{ \hat \bR^{h} \in \cH } \sum_{(u,i) \in \cD } (e_{u,i} -\hat e_{u,i}^{h})^2 /\hat p_{u,i}^2$. 
\end{lemma}

\section{Proposed Methods} 


\subsection{Motivation}\label{dr_concerns}

We reveal some worrying features of DR method, which provides an initial motivation.   
Lemma~\ref{lemma1} formally gives the bias and variance of the DR estimator.
According to the lemma, {$Bias[\cL_{DR}(\hat \bR, \bR^{o})] \approx 0$}, if either $(\hat e_{u,i} - e_{u,i}) \approx 0$ or $(\hat p_{u,i} - p_{u,i}) \approx  0$, which is the property of double robustness. 
Nonetheless, both the bias and variance terms still have some issues.
Specifically, the bias consists of the product of the errors of the propensity score model and imputation model weighted by $1/\hat p_{u,i}$. 
The term $(e_{u,i}-\hat e_{u,i}) / \hat p_{u,i}$ is worrisome, as  $1/\hat p_{u,i}$ tends to be large in unclicked events and inaccurate estimates of $e_{u,i}$ are most likely to occur in these events.
Analogously,  $( \hat e_{u,i} - e_{u,i} )^{2} /  \hat p^{2}_{u,i}$ in the variance term is also likely to be problematic.
 
 It can be seen that both the bias and variance 
 are 
correlated with the term of error deviation $| \hat e_{u,i} - e_{u,i} |$. Thus, it may be helpful to reduce them if the magnitude of error deviation is small. This is the basic idea of DR-JL approach that tries to reduce the error deviations of all events by optimizing the loss function (\ref{drjl_loss}). Further, the MRDR method~\cite{MRDR_DL} proposed replacing
 {$\cL_{e}^{DR-JL}(\theta)$} in (\ref{drjl_loss}) with {$\cL_{e}^{MRDR}(\theta)$} in (\ref{mrdr_loss}) to deal with the large variance term. The idea behind Equation (\ref{mrdr_loss}) is the truth that  
    {
	 \[   \V_{\cO}[ \cL_{DR}(\hat \bR, \bR^{o})  ]  =    \frac{1}{ |\cD|^{2} } \sum_{(u,i)\in \cD}  \mathbb{E}_{\cO} [ \frac{ o_{u,i} (1- p_{u,i}) (  \hat e_{u,i} - e_{u,i} )^{2}  }{    \hat p_{u,i}^{2}  } ],	\]
	 }
namely, the expectation of {$\cL_{e}^{MRDR}(\theta)$} equals to {$\V_{\cO}[ \cL_{DR}(\hat \bR, \bR^{o})]$}.  

Interestingly, Lemma~\ref{lemma2} 
shows that the generalization bound 
depends on a weighted sum of the bias term and square root of the variance term in addition to the empirical 
loss, which fully reflects the feature of bias-variance trade-off. Since DR-JL does not control the bias and variance directly and MRDR pays no attention to the bias, both of them may still have poor generalization performance.  

\subsection{A Generalized DR Learning Framework}

The difference between DR-JL and MRDR lies in the loss function of the error imputation model. As presented in Section~\ref{existing_methods}, the alternating algorithm of DR-JL implies that its underlying loss is {$\cL_{DR}(\hat \bR, \bR^o) + \cL_e^{DR-JL}(\theta).$}
Similarly, the real loss of MRDR is {$\cL_{DR}(\hat \bR, \bR^o) + \cL_e^{MRDR}(\theta)$}.
 Note that the real loss functions of DR-JL and MRDR share a similar structure, so they can be discussed within a generalized framework. 
 The real loss function of this framework has the following form
    {
    \begin{equation} \cL(\hat \bR, \bR^o) + Metric\{ \cL(\hat \bR, \bR^o) \},              \end{equation}
    }
where {$\cL(\hat \bR, \bR^o)$} is an arbitrary unbiased loss function for training CVR prediction model, such as {$\cL_{IPS}(\hat \bR, \bR^o)$}, {$\cL_{EIB}(\hat \bR, \bR^o)$} and {$\cL_{DR}(\hat \bR, \bR^o)$}. {$ Metric\{ \cL(\hat \bR, \bR^o) \}$} is a  pre-specified metric that reflects some features of {$\cL(\hat \bR, \bR^o)$} and is usually applied to learn the error imputation model.      
For example, the MRDR chooses {$\V_{\cO}[ \cL_{DR}(\hat \bR, \bR^{o})]$} as the metric, and DR-JL uses $\sum_{(u,i)\in \cD} (\hat e_{u,i} - e_{u,i})^2$. 
Table 2 summarizes the metrics and ideas  of existing doubly robust methods and our proposed methods, DR-BIAS and DR-MSE, which will be detailedly illustrated in Section~\ref{new_methods}.   
\begin{table}[h!]
\scalebox{0.85}{
\begin{threeparttable}  
\caption{Generalized framework of various DR methods}
\vspace{-0.3cm}
\centering 
\begin{tabular}{l l l}    
\hline
			    Method   	 &    Metric  &  Goal   \\
\hline
                  DR-JL  &  $\sum_{ (u,i)\in \cD } (\hat e_{u,i} - e_{u,i})^2$     & Control error of imputation.        \\
	   	          MRDR      &   $\V_{\cO}[ \cL_{DR}(\hat \bR, \bR^{o})]$  & Control variance.             \\
                  DR-BIAS  & $Bias[  \cL_{DR}(\hat \bR, \bR^{o})]$     &  Further reduce  bias.     \\ 
DR-MSE    &   $MSE[\cL_{DR}(\hat \bR, \bR^{o})]$   &  Bias-variance trade-off.  \\ 
                   \hline  
	\end{tabular}
	\end{threeparttable}}
	\vspace{-0.2cm}
\end{table}

It is noteworthy that due to the missing $r_{u,i}(1)$,
optimizing {$Metric\{ \cL(\hat \bR, \bR^o) \}$} directly is sometimes not feasible. In this case, one can use an approximation of {$Metric\{ \cL(\hat \bR, \bR^o) \}$}. For example, DR-JL adopts the feasible 
loss function~(\ref{drjl_loss}) to approximate the infeasible $\sum_{ (u,i)\in \cD } (\hat e_{u,i} - e_{u,i})^2$,  and MRDR employs~(\ref{mrdr_loss}) to substitute {$\V_{\cO}[ \cL_{DR}(\hat \bR, \bR^{o})]$}.

Importantly, the proposed framework 
provides a valuable opportunity to develop a series of new unbiased CVR estimators with different characteristics to accommodate different application scenarios. In Section~\ref{new_methods}, we will develop two new DR approaches based on this framework. 

\subsection{Two New DR Methods}\label{new_methods}
As discussed in Section~\ref{dr_concerns}, MRDR aims to reduce the variance of {$\cL_{DR}(\hat \bR, \bR^o)$}, and is expected to achieve a more robust performance. However, this strategy works well only when {$Bias[ \cL_{DR}(\hat \bR, \bR^o) ]$} is small enough as suggested by the generalization bound presented in Lemma~\ref{lemma2}. Reducing variance  is less effective when the bias is large. DR-JL attempts to lower both the bias and variance  
by reducing the error deviation of the imputation model.  
Nevertheless, it does not directly control the bias and variance of  {$\cL_{DR}(\hat \bR, \bR^o)$}. 
To alleviate these limitations, we propose two new DR methods, DR-BIAS and DR-MSE, which are designed to further reduce bias and achieve better bias-variance trade-off, respectively.

\subsubsection{DR-BIAS}\label{dr-bias}
DR-BIAS aims at further reducing the bias of the typical DR method through the optimization of the imputation model, since an accurate CVR prediction means that the bias should be small enough. Based on Lemmas~\ref{lemma1} and~\ref{lemma2}, we design a variant of the bias of DR method as the metric to achieve this goal, given by 
        {
		\[      \frac{1}{ | \cD | }  \sum_{(u,i) \in D}  \frac{ (o_{u,i} - \hat p_{u,i})^{2}  }{ \hat p^{2}_{u,i} } ( e_{u,i} - \hat e_{u,i})^{2}.          \]}
However, the above metric is infeasible due to the missing of $e_{u,i}$ in unclicked events. We make an approximation of it and define the loss of the imputation model of DR-BIAS as follows
{
\begin{equation} \label{bias_loss} 
\cL_{e}^{DR-BIAS}(\theta) = \sum_{(u,i)\in \cD} \frac{ o_{u,i} (  \hat e_{u,i} - e_{u,i} )^{2}  }{  \hat p_{u,i} } \cdot  \frac{  (o_{u,i} - \hat p_{u,i})^{2}   }{ \hat p_{u,i}^2 }.  \end{equation}
}

By a comparison between Equation (\ref{mrdr_loss}) and (\ref{bias_loss}), we find that
 (\ref{bias_loss}) just substitutes the weight $(1-\hat p_{u,i})/\hat p_{u,i}$ with $(1 - \hat p_{u,i})^2/\hat p_{u,i}^2$ in clicked events. Also note that 
 {
    \begin{align*}
    \begin{cases}
      & (1-\hat p_{u,i})/\hat p_{u,i}    >1, \text{ if } \hat p_{u,i} < 1/2, \\
      & (1-\hat p_{u,i})/\hat p_{u,i}   <1, \text{ if } \hat p_{u,i} > 1/2,
    \end{cases}
    \end{align*}}
which means that DR-BIAS further magnifies  
the penalty of the clicked 
events with low propensity, and minifies those with high propensity. This leads to a desired effect:  
in the clicked events that the propensity model performs poorly, the amplified weights force the error imputation model to perform well. In other words, error imputation model complements the inaccurate part of the propensity score model.   
Thus, DR-BIAS would have smaller bias than other methods.

\subsubsection{DR-MSE}
Lemma~\ref{lemma2} indicates that pursuing the bias reduction  or variance reduction alone cannot fully control the generalization error. 
Seeking a better balance between the bias and variance appears to be a more effective way to improve the prediction accuracy. Therefore, we design a new model, namely DR-MSE, to achieve this goal. Specifically, a generalized Mean Squared Error (MSE) metric for DR-MSE method is defined as 
    {
    \begin{equation}\label{mse_loss}
     	\cL_{e}^{DR-MSE}(\theta) = \lambda\cL_{e}^{DR-BIAS}(\theta) + 	(1-\lambda)\cL_{e}^{MRDR}(\theta),
    \end{equation}}
where $\lambda$ is a hyper-parameter for controlling the strength of the bias term and the variance term. When $\lambda=1$, DR-MSE is reduced to DR-BIAS; when $\lambda=0$, DR-MSE is reduced to MRDR; when $\lambda=0.5$, DR-MSE optimizes the MSE of {$\cL_{DR}(\hat \bR, \bR^{o})$} scaled by 0.5 through the imputation model.

\textcolor{black}{However, simply using a hyper-parameter $\lambda$ for all samples is not flexible enough due to the different characteristics and popularities of users and items.  
Specifically, different samples suffer from different issues during training, i.e., some might have higher variance while others might have worse bias. 
Thus, it is necessary to adopt different bias-variance tradeoff strategies for different user-item pairs. To achieve this goal, $\lambda$ can be computed through a function $\lambda_\xi(x_{u,i})$ parameterized by $\xi$, such as a neural network, which enables personalized values for different user-item pairs. The improved loss of DR-MSE is as follows
{
 \begin{equation}\label{drmsev2}
 \begin{array}{lll}
\cL_{e}^{DR-MSE}\left(\theta, \lambda_\xi\right)=\sum\limits_{(u,i)\in \cD} \frac{ o_{u,i}\lambda_\xi(x_{u,i}) (  \hat e_{u,i} - e_{u,i} )^{2}  }{  \hat p_{u,i} } \cdot  \frac{  (o_{u,i} - \hat p_{u,i})^{2}   }{ \hat p_{u,i}^2 }\\
\qquad\qquad\qquad\quad\;\, +\sum\limits_{(u,i)\in \cD} \frac{ o_{u,i}(1-\lambda_\xi(x_{u,i})) (  \hat e_{u,i} - e_{u,i} )^{2}  }{    \hat p_{u,i}  } \cdot  \frac{ 1- \hat p_{u,i}  }{ \hat p_{u,i} }.
 \end{array}
 \end{equation}
}
}

Essentially, the generalization bound of DR methods contains a weighted sum of the bias term and square root of the variance term, which can be flexibly tradeoff via the proposed generalized MSE metric in~(\ref{mse_loss}). Thus, it is expected that DR-MSE can obtain a better prediction performance under the tighter generalization bound.

\begin{figure}
    \centering
    \includegraphics[width=0.45\textwidth]{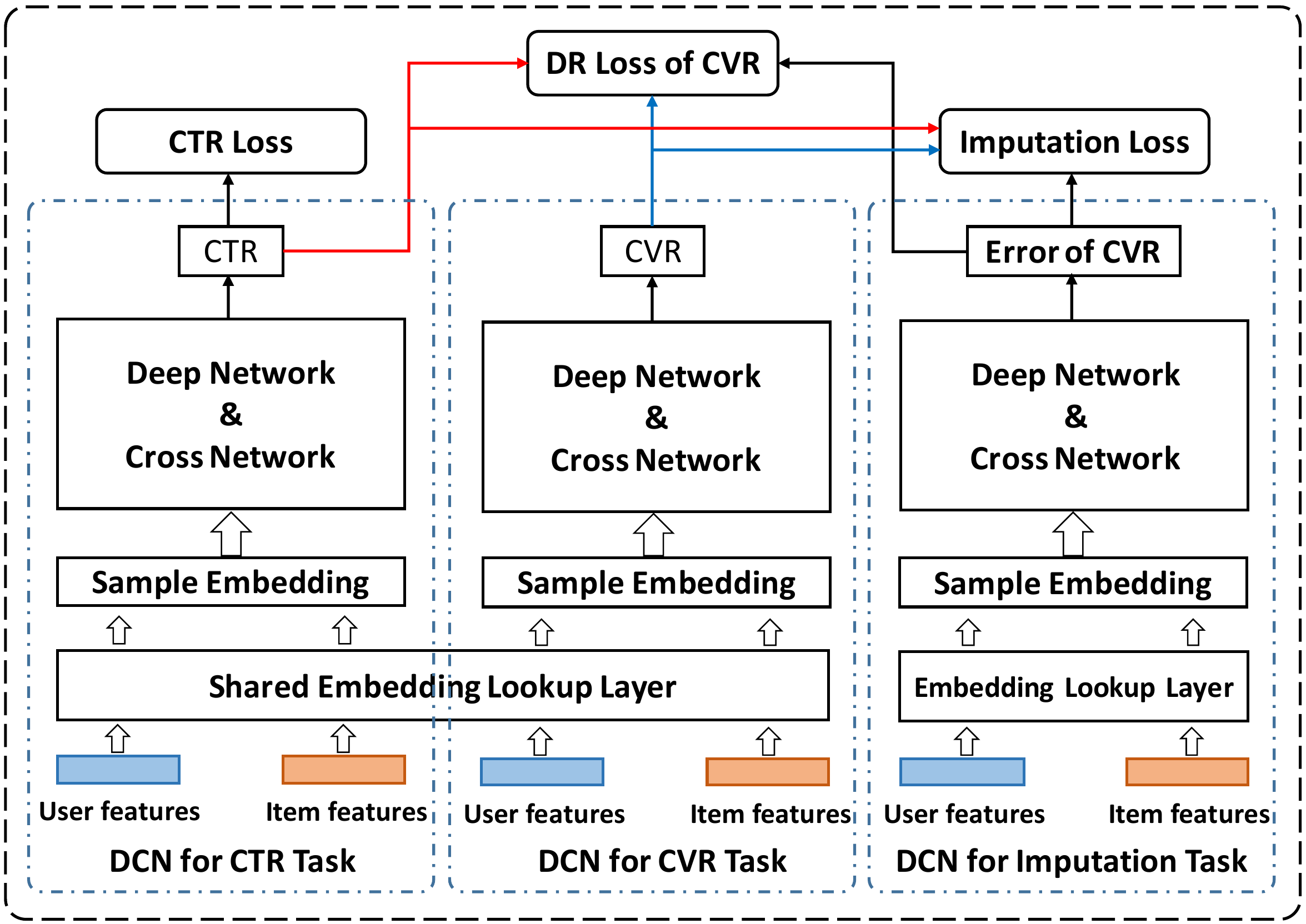}
    \vspace{-0.2cm}
    \caption{Model architecture of DR-BIAS and DR-MSE for experiments on large-scale industrial dataset. DCN is used as the base model for feature interaction learning for illustration only, and it can be readily replaced with other models such as FM~\cite{rendle2010factorization}, Wide\&Deep~\cite{cheng2016wide} and DeepFM~\cite{guo2017deepfm}.}\label{fig:causalmtl4.1}
    \vspace{-0.2cm}
\end{figure}

\section{Proposed Training Approach}
\subsection{Model Architecture and Training Objective}
Figure~\ref{fig:causalmtl4.1} shows the architecture of DR-BIAS and DR-MSE for experiments on real industrial scenarios.
It is a multi-task learning framework with three DCN networks for the predication of post-view click-through rate (CTR), CVR, and error imputation, respectively. The embedding lookup layers of the DCN models for the CTR and CVR tasks are shared to tackle data sparsity issue, while the DCN model for the error imputation has its own embedding lookup layer. Note that DCN can be readily replaced with other models such as FM~\cite{rendle2010factorization},  Wide\&Deep~\cite{cheng2016wide} and DeepFM~\cite{guo2017deepfm}. We evaluate our proposed methods with both FM and DCN in our experiments.

During optimization, the CTR, CVR, and error imputation models are updated alternatively with stochastic gradient descent. Specifically, with the parameters of both CTR and CVR models fixed, the error imputation model is updated first by optimizing (\ref{drmsev2}).
With model parameters of the error imputation model fixed, the CTR and CVR models are optimized jointly through the sum of CVR loss and CTR loss 
$$
\begin{aligned}
\cL_{\text{CTCVR}}\left(\phi, \zeta, \theta(\lambda_\xi) \right)=& \cL_{DR}(\phi, \theta(\lambda_\xi)) + \cL_{CTR}(\zeta), 
\end{aligned}
$$ 
where $\cL_{DR}(\phi, \theta(\lambda_\xi))=\sum_{(u,i) \in \cD} [ \hat e_{u,i}  +   o_{u,i} (e_{u,i} -  \hat e_{u,i}) / \hat p_{u, i}]$, $\cL_{CTR}(\zeta) = - \sum_{(u,i) \in \cD}  [o_{u, i} \cdot \log (\hat p_{u,i})+(1-o_{u, i}) \cdot \log (1-\hat p_{u,i}) ]$, $\hat p_{u,i}$ is the predicted CTR value, and used as the estimated propensity for unbiased CVR estimation. 
This joint learning process continues until the model converges. For DR-MSE, the optimization process 
involves updating $\lambda_\xi(\cdot)$, which makes it more challenging. 
In Section~\ref{tri-level-opt}, we formally formulate the optimization problem and propose an effective training algorithm.
\subsection{Tri-Level DR-MSE Joint Learning (JL) Optimization and Training Algorithm
}\label{tri-level-opt}

We propose the tri-level optimization DR-MSE JL approach shown in Figure \ref{fig:DR-MSE}. Compared to the existing joint learning methods, our approach allows adaptively updating the $\lambda_\xi$ in DR-MSE. This goal can be formalized as the following tri-level optimization problem
{
$$
\begin{aligned}
\xi^{*} &=\arg \min _{\xi} \cL_{DR}\left(\phi^{*}(\theta^{*}(\lambda_\xi)), \zeta^{*}(\theta^{*}(\lambda_\xi))\right) \\
\text { s.t. } \phi^{*}(\theta^{*}(\lambda_\xi)), \zeta^{*}(\theta^{*}(\lambda_\xi))&=\arg \min _{\phi, \zeta} \cL_{\text{CTCVR}}\left(\phi, \zeta, \theta^{*}(\lambda_\xi) \right)\\
\text { s.t. } \theta^{*}(\lambda_\xi)&=\arg \min _{\theta} \cL_{e}^{DR-MSE}\left(\theta, \lambda_\xi\right)
\end{aligned}
$$}

\begin{figure}
    \centering
    \includegraphics[width=0.46\textwidth]{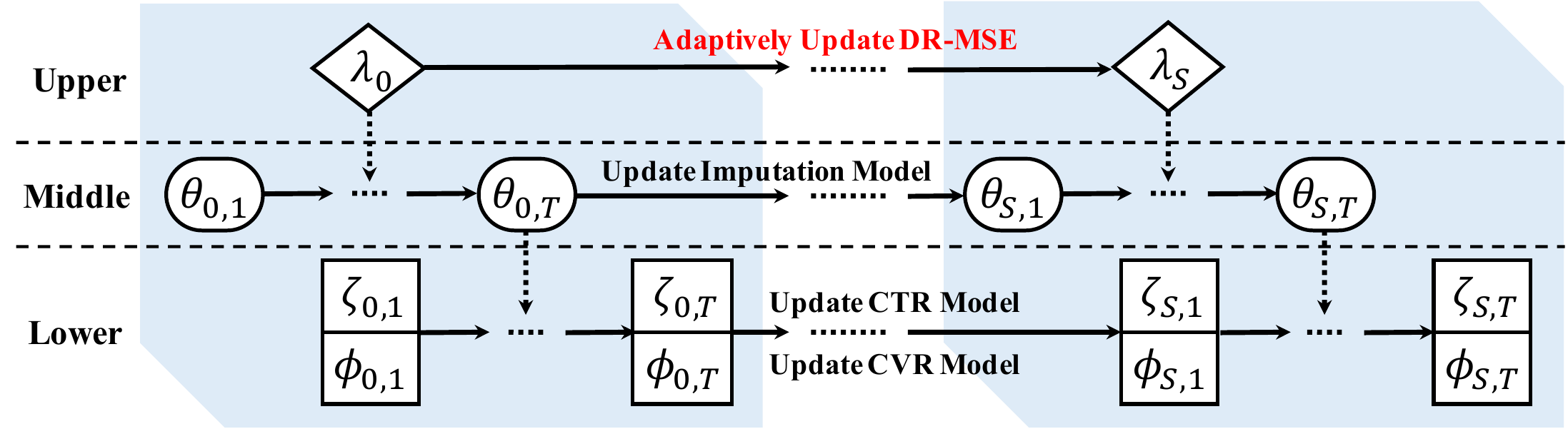} 
    \vspace{-0.2cm}
    \caption{The proposed tri-level DR-MSE joint learning optimization method updates the $\lambda_s$ in DR-MSE adaptively in upper level, while the existing bias-variance tradeoff approach uses a fixed $\lambda_0$. Middle and lower levels are for the joint training between the CTR\&CVR and error imputation models.}
    \label{fig:DR-MSE} 
    \vspace{-0.2cm}
\end{figure}

There are two challenges for solving the above problem. Firstly, it is computationally expensive to search for the optimal DR-MSE by minimizing the upper loss in the tri-level DR-MSE JL optimization method. Secondly, the DR-MSE parameter $\lambda_\xi$ of the upper model is difficult to be minimized as there is no closed-form solution. To address them, we further propose a training algorithm for this tri-level optimization problem as shown in Alg. 1. 

\begin{algorithm}[t]
\caption{Tri-Level DR-MSE JL Optimization Training}
\LinesNumbered 
\KwIn{$S$, observed ratings $\mathbf{R}^{o}$, learned propensities $\hat{\mathbf{P}}$, }
\For{$\mathcal{O}_{s}^l, \mathcal{O}_{s}^u \subset \mathcal{O}$ and $\mathcal{D}_{s} \subset \mathcal{D}$ ($s\in\{0, 1, \cdots, S-1\}$)}{
Compute an update function based on $\mathcal{O}_{s}^l$\
$\theta_{s+1}\left(\lambda_{s}\right) \leftarrow \theta_{s}-\eta \nabla_{\theta_s} \cL_{e}^{DR-MSE}\left(\theta, \lambda_{s}\right)$\;
Compute an update function based on $\mathcal{D}_{s}$\ $\phi_{s+1}(\theta_{s+1}\left(\lambda_{s}\right)) \leftarrow \phi_s-\eta \nabla_{\phi_s}\cL_{\text{CTCVR}}\left(\phi, \zeta_s,  \theta_{s+1}\left(\lambda_{s}\right)\right)$\;
Compute an update function $\zeta_{s+1}(\theta_{s+1}\left(\lambda_{s}\right)) \leftarrow \zeta_s-\eta \nabla_{\zeta_s}\cL_{\text{CTCVR}}\left(\phi_s, \zeta,  \theta_{s+1}\left(\lambda_{s}\right)\right)$\;
Compute the upper loss based on $\mathcal{O}_{s}^u$ \\ $\cL_{DR}\left(\phi_{s+1}(\theta_{s+1}\left(\lambda_{s}\right)), \zeta_{s+1}(\theta_{s+1}\left(\lambda_{s}\right))\right)$\;
Update the bias-variance trade off parameter \\
$\xi_{s+1} \leftarrow \xi_{s}-\eta \nabla_{\xi_{s}} \cL_{DR}\left(\phi_{s+1}(\theta_{s+1}(\lambda_{\xi})), \zeta_{s+1}(\theta_{s+1}(\lambda_{\xi}))\right)$\;
Update the bias-variance trade off model $\lambda_{s+1} \leftarrow \lambda_{\xi_{s+1}}$\;
\For{$\mathcal{O}_{s, t}^{l} \subset \mathcal{O}$ ($t\in\{0, 1, \cdots, T-1\}$)}{
Update the imputation model based on $\mathcal{O}_{s, t}^l$ $\theta_{s,t+1} \leftarrow \theta_{s,t}-\eta \nabla_{\theta_{s,t}} \cL_{e}^{DR-MSE}\left(\theta, \lambda_{s+1}\right)$\;}
\For{$\mathcal{D}_{s, t} \subset \mathcal{D}$ ($t\in\{0, 1, \cdots, T-1\}$)}{
Update the propensity model based on $\mathcal{D}_{s, t}$\ $\zeta_{s,t+1} \leftarrow \zeta_{s,t}-\eta \nabla_{\zeta_{s,t}}\cL_{\text{CTCVR}}\left(\phi_{s,t}, \zeta,  \theta_{s,T}\right)$\;
Update the predication model based on $\mathcal{D}_{s, t}$\ $\phi_{s,t+1} \leftarrow \phi_{s,t}-\eta \nabla_{\phi_{s,t}}\cL_{\text{CTCVR}}\left(\phi, \zeta_{s,t},  \theta_{s,T}\right)$\;}
Copy the model parameter $\theta_{s+1,0} \leftarrow \theta_{s,T}$\;
Copy the propensity model's parameter $\zeta_{s+1,0} \leftarrow \zeta_{s,T}$\;
Copy the predication model's parameter $\phi_{s+1,0} \leftarrow \phi_{s,T}$.
}
\end{algorithm}

For illustration purposes, the relevant parameters in Alg. 1 are updated using vanilla SGD. In practice, both SGD and its variants can be used for iterative updates. Specifically, for the error imputation optimization problem, $\cL_{e}^{DR-MSE}$ is differentiable w.r.t. the parameter $\theta$ of the error imputation model.  \textcolor{black}{Given $\lambda_s$, one can compute the value of $\theta_{s+1}\left(\lambda_{s}\right)$ after a single vinalla SGD. It should be noted that this value is not directly used for the update of the error imputation model parameter $\theta_{s+1}$. Moreover, the reason for using single-step SGD is that multi-step SGD here does not result in better performance, but rather increases the computational complexity~\cite{jenni2018deep}.}

Similarly, $\cL_{\text{CTCVR}}$ is differentiable w.r.t. both CTR model parameter $\zeta$ and CVR model parameter $\phi$. Given the pseudo-updated $\theta_{s+1}\left(\lambda_{s}\right)$, one can compute the value of $\zeta_{s+1}\left(\theta_{s+1}\left(\lambda_{s}\right)\right)$ and $\phi_{s+1}\left(\theta_{s+1}\left(\lambda_{s}\right)\right)$ after a single vanilla SGD. Both of the values are not directly used for the update of the CTR and CVR model as well. After that, with given $\zeta_{s+1}\left(\theta_{s+1}\left(\lambda_{s}\right)\right)$ and $\phi_{s+1}\left(\theta_{s+1}\left(\lambda_{s}\right)\right)$, we update the bias-variance tradeoff parameter in DR-MSE from $\lambda_s$ to $\lambda_{s+1}$ via a single vanilla SGD. Finally, based on the updated $\lambda_{s+1}$, we take the idea of joint learning to update the error imputation model parameter $\theta_{s+1}$ and the CTR\&CVR model parameters $\zeta_{s+1}, \phi_{s+1}$, in which the classical multi-step SGD is used until the stopping criteria is satisfied.


\section{Real-world Experiments}
In this section,  we evaluate the proposed methods by conducting experiments on three real-world datasets, including two benchmark datasets with missing-at-random (MAR) ratings and one large-scale industrial product dataset. We aim to answer the following two research questions (RQ): (1) How do our methods compare with state-of-the-art models in terms of debiasing performance in practice? (2) How do the bias-variance tradeoff and the modeling of unobserved data affect the performance of the proposed methods in practice?

\subsection{Experimental Setup}
\subsubsection{Datasets with MAR Ratings}
A MAR testing set is important for assessing the performance of an unbiased recommender. Thus, we follow existing studies~\cite{MRDR_DL,RecSys_Saito20} to use {\bf Coat Shopping\footnote{https://www.cs.cornell.edu/\textasciitilde schnabts/mnar/}} and {\bf Yahoo! R3\footnote{http://webscope.sandbox.yahoo.com/}} for the evaluation of CVR prediction model. To make the two datasets consistent with the CVR prediction task, we further preprocess them following previous studies~\cite{RecSys_Saito20,MRDR_DL}. The detailed descriptions of these two datasets and the corresponding data preprocessing method are provided in Appendix~\ref{apdx-data}.

\subsubsection{Industrial Product Dataset.}
To provide more comprehensive and reliable evaluation, we also conduct experiments on a large-scale App advertising dataset collected from a real-world system. We denote this dataset as \textbf{\textbf{Product}} with some statistics of it displayed in Table~\ref{tab:huawei}. It contains 8 consecutive days logged data from the system, with the first 7 days for training and the last day for testing. Each sample of the dataset contains features from a user, an item and the corresponding context. Although the unbiased data in CVR prediction is unobtainable in real applications since we cannot force users to randomly click the exposed items, the experiments can still provide valuable observations for the applications of debiasing CVR prediction models in real systems.

\subsubsection{Baselines and Implementation}
For experiments on \textbf{Coat} and \textbf{Yahoo}, we compare our methods with several competitive baselines, including Naive, IPS, DR-JL and MRDR. The base model for all methods is factorization machine. Some brief descriptions of them and implementation details are provided in Appendix~\ref{apdx-semi}. 
For experiments on \textbf{Product}, we also select some state-of-the-art CVR prediction models for large datasets, including DCN~\cite{wang2017deep}, ESMM~\cite{ESMM18}, Multi\_IPW~\cite{Multi_IPW} and Multi\_DR~\cite{Multi_IPW}. The base model for all methods is DCN. More details are provided in Appendix~\ref{apdx-product}.

\subsubsection{Experimental Protocols}
For experiments on \textbf{Coat} and \textbf{Yahoo}, we evaluate the ranking performance with two types of metrics, i.e., discounted cumulative gain (DCG) and recall, as prior work on debiasing CVR prediction~\cite{RecSys_Saito20,MRDR_DL}.
For experiments on \textbf{Product}, we evaluate our proposed methods on three important tasks, i.e., CTR, CVR, and CTCVR ($CTCVR=CTR*CVR$) predictions, with the AUC score following existing works~\cite{ESMM18,Multi_IPW}.

\begin{table}[t]
\caption{Statistics of the advertising dataset \textbf{Product}}
\vspace{-0.3cm}
\center
\small
\renewcommand\arraystretch{1.0}
\setlength{\tabcolsep}{8.pt}
\scalebox{0.85}{
\begin{tabular}{cccccc}
\hline
Dataset & \#Impression & \#Click & \#Conversion & \#User & \#Item\\
\hline
Training & 739.66M    & 3.73M   & 1.90M         & 524K &   68K\\
Testing & 99.73M     & 519K    & 268K          & 283K &   52K\\
\hline
\end{tabular}}
\begin{tablenotes}
\footnotesize
\item Note: ``M'' means million, and ``K'' means thousand. 
\end{tablenotes}
\label{tab:huawei}   
\vspace{-0.3cm}
\end{table}

\begin{table*}[t]
\caption{Performance comparison based on \textbf{Coat} and \textbf{Yahoo}. 
}
\vspace{-0.3cm}
\center
\small
\renewcommand\arraystretch{1.0}
\setlength{\tabcolsep}{5.pt}
\begin{threeparttable}  
\scalebox{0.90}{
\begin{tabular}{c|c|ccc|ccc}
\hline
Datasets                & Models           & \multicolumn{1}{c}{DCG@2}               & \multicolumn{1}{c}{DCG@4}               & \multicolumn{1}{c|}{DCG@6}               & \multicolumn{1}{c}{Recall@2}            & \multicolumn{1}{c}{Recall@4}            & \multicolumn{1}{c}{Recall@6}            \\
\hline
\multirow{7}{*}{\textbf{Coat}}  & Naïve           & 0.7283 $\pm$ 0.0264          & 0.9763 $\pm$ 0.0258          & 1.1512 $\pm$ 0.0241          & 0.8474 $\pm$ 0.0310          & 1.3786 $\pm$ 0.0374          & 1.8490 $\pm$ 0.0379          \\
                       & IPS             & 0.7102 $\pm$ 0.0220          & 0.9596 $\pm$ 0.0222          & 1.1299 $\pm$ 0.0210          & 0.8248 $\pm$ 0.0272          & 1.3596 $\pm$ 0.0360          & 1.8174 $\pm$ 0.0377          \\
                       & DR-JL           & 0.7416 $\pm$ 0.0224          & 1.0021 $\pm$ 0.0224          & 1.1762 $\pm$ 0.0229          & 0.8645 $\pm$ 0.0264          & 1.4225 $\pm$ 0.0362          & 1.8906 $\pm$ 0.0403          \\
                       & MRDR & 0.7442 $\pm$ 0.0225          & 1.0132 $\pm$ 0.0219          & 1.1947 $\pm$ 0.0194          & 0.8736 $\pm$ 0.0273          & 1.4494 $\pm$ 0.0325          & 1.9370 $\pm$ 0.0318          \\
                       \cline{2-8}
                       & DR-BIAS        & \textbf{0.7648 $\pm$ 0.0192*} & \textbf{1.0353 $\pm$ 0.0169*} & \textbf{1.2127 $\pm$ 0.0162*} & \textbf{0.8959 $\pm$ 0.0251*} & \textbf{1.4751 $\pm$ 0.0273*} & \textbf{1.9517 $\pm$ 0.0324*} \\
                       & DR-MSE          & \textbf{0.7682 $\pm$ 0.0151*} & \textbf{1.0401 $\pm$ 0.0150*} & \textbf{1.2170 $\pm$ 0.0139*} & \textbf{0.8997 $\pm$ 0.0194*} & \textbf{1.4816 $\pm$ 0.0241*} & \textbf{1.9569 $\pm$ 0.0262*} \\
                       \hline\hline
\multirow{7}{*}{\textbf{Yahoo}} & Naïve           & 0.5469 $\pm$ 0.0009          & 0.7466 $\pm$ 0.0008          & 0.8714 $\pm$ 0.0004 & 0.6479 $\pm$ 0.0012          & 1.0745 $\pm$ 0.0016          & 1.4098 $\pm$ 0.0013 \\
                       & IPS             & 0.5502 $\pm$ 0.0010          & 0.7520 $\pm$ 0.0009          & 0.8751 $\pm$ 0.0009          & 0.6545 $\pm$ 0.0017          & 1.0797 $\pm$ 0.0017          & \textbf{1.4168 $\pm$ 0.0019} \\
                       & DR-JL           & 0.5602 $\pm$ 0.0034          & 0.7586 $\pm$ 0.0030          & 0.8808 $\pm$ 0.0025          & 0.6615 $\pm$ 0.0042          & 1.0849 $\pm$ 0.0049          & 1.4129 $\pm$ 0.0039          \\
                       & MRDR & 0.5623 $\pm$ 0.0024          & 0.7603 $\pm$ 0.0027          & 0.8820 $\pm$ 0.0020          & 0.6646 $\pm$ 0.0033          & 1.0881 $\pm$ 0.0045          & 1.4145 $\pm$ 0.0037          \\
                       \cline{2-8}
                       & DR-BIAS        & \textbf{0.5646 $\pm$ 0.0023*} & \textbf{0.7624 $\pm$ 0.0021*} & \textbf{0.8841 $\pm$ 0.0018*}          & \textbf{0.6676 $\pm$ 0.0026*} & \textbf{1.0904 $\pm$ 0.0028*} & \textbf{1.4169 $\pm$ 0.0020}          \\
                       & DR-MSE          & \textbf{0.5662 $\pm$ 0.0017*} & \textbf{0.7639 $\pm$ 0.0016*} & \textbf{0.8850 $\pm$ 0.0014*} & \textbf{0.6670 $\pm$ 0.0026*} & \textbf{1.0891 $\pm$ 0.0029} & 1.4140 $\pm$ 0.0028          \\
                       \hline
\end{tabular}
}   
\end{threeparttable}
\begin{tablenotes}
\footnotesize
\item \qquad\qquad Note: * statistically significant results ($\text{p-value} \leq 0.05$) using the paired-t-test compared with the best baseline.
\end{tablenotes}
\label{tab:real-result}  
\vspace{-0.4cm}
\end{table*}

\begin{table}[t]
\caption{Performance comparison based on \textbf{Product}. 
}
\vspace{-0.3cm}
\center
\small
\renewcommand\arraystretch{1.0}
\setlength{\tabcolsep}{5.pt}
\begin{threeparttable}  
\scalebox{0.9}{
\begin{tabular}{cccc}
\hline
Models           & CTR AUC (\%)         & CVR AUC (\%)     & CTCVR AUC (\%)       \\
\hline
DCN              & 90.763          & 75.691           & 95.254          \\
ESMM             & 90.704          & 81.647          & 95.505          \\
DR-JL           & 90.754          & 81.768          & 95.548          \\
Multi\_IPW       & 90.794          & 81.912          & 95.571          \\
Multi\_DR        & 90.807          & 81.864          & 95.569          \\
MRDR             & 90.721          & 81.810          & 95.535          \\
\hline
DR-BIAS         & \textbf{90.913} & \textbf{81.974} & \textbf{95.633} \\
DR-MSE          & \textbf{90.825} & \textbf{82.067} & \textbf{95.654} \\
\hline
\end{tabular}
}   
\end{threeparttable}    
\label{tab:huawei-result}  
\vspace{-0.3cm}
\end{table}

\subsection{Overall Performance (RQ1)}

\subsubsection{Unbiased Evaluation.} 
The experimental results on \textbf{Coat} and \textbf{Yahoo} are shown in Table~\ref{tab:real-result}. We have the following observations.

First, our proposed methods are effective for debiasing CVR prediction task. As shown in Table~\ref{tab:real-result}, both DR-MSE and DR-BIAS consistently outperform all the other ones in terms of DCG@K and Recall@K ($K=2,4,6$) on the two real-world datasets, with only one exception of DR-MSE on Recall@6 of \textbf{Yahoo}. In particular, DR-MSE achieves a significant 3.22\%, 2.65\% and 1.87\% relative improvements over MRDR on DCG@2, DCG@4 and DCG@6, respectively.

Second, it is necessary to improve the bias and variance of the typical DR method under inaccurate propensity estimation and error imputation so as to enhance its robustness and ranking performance.
As shown in Table~\ref{tab:real-result}, IPS has worse performance on \textbf{Coat} and only comparable performance on \textbf{Yahoo} compared with the Naive method, since it suffers heavily from the high variance issue. Both DR-JL and MRDR performs better compared with IPS because of their double robustness. 
DR-BIAS improves over MRDR by achieving smaller bias through magnifying the penalty of the clicked events with low propensity while minifying those with high propensity as analyzed in Section~\ref{dr-bias}.
However, these DR methods still suffer from the high bias and/or variance issues.
Our proposed DR-MSE can further achieve improvements over all other DR methods by better controlling the bias and variance.

\subsubsection{Large-scale Industrial Dataset.} The experimental results on \textbf{Product} are shown in Table~\ref{tab:huawei-result}. Firstly, we can observe that ESMM improves over DCN on CVR and CTCVR prediction tasks by tackling the data sparsity issue with the multi-task learning framework, but it still suffers from the selection bias issue. Secondly, the debiasing CVR models can simultaneously tackle the data sparsity and selection bias issues, thus they outperform DCN and ESMM.
Thirdly, our proposed methods achieve significant improvements over existing debiasing CVR prediction models, including DR-JL, Multi\_IPW, Multi\_DR and MRDR, which is consistent with the observations on experiments with unbiased evaluation. It demonstrates that our proposed methods have both theoretical guarantee and great application potentials in real industrial systems. 

\subsection{In-depth Analysis of DR-MSE (RQ2)}
We conduct an analysis of two important aspects of DR-MSE with \textbf{Coat} in this section.
The experimental results are displayed in Figure~\ref{fig:hyperparameter}. Note that similar results can be observed on other datasets, and we do not present them here only due to space limitations.

The loss of the imputation model of DR-MSE contains a bias term and a variance term. We conduct experiments by manually varying $\lambda$ in Eq. (\ref{mse_loss}) to demonstrate the necessity of conducting bias-variance tradeoff. The left part of Figure~\ref{fig:hyperparameter} presents the experimental results of DR-MSE when varying $\lambda$ from 0.1 to 0.9. We can find that the performance of DR-MSE first improves with the increase of $\lambda$, and then gradually drops.
It shows that an appropriate tradeoff between this two terms can improve model generalization performance.

DR methods can achieve double robustness by jointly considering clicked events and unclicked events. Here, we also study the effect of the sample ratio of unclicked events to clicked events on the performance of DR-MSE.
When the sample ratio is set to ``All'', all the unclicked events are utilized for training; when the sample ratio is set to 0, only clicked events are utilized. As shown in Figure~\ref{fig:hyperparameter}, when the sample ratio ranges from 0 to ``All'', the DCG@K ($K=1,3,5$) scores on \textbf{Coat} show an apparent increase first, and then tend to saturate or decrease slightly. It suggests that a certain amount of unclicked events can provide useful information for improving the prediction model with the assistance of an imputation model, but further improvement is marginal when passing some threshold. In real advertising applications, the unclicked events are usually composed of the exposed but unclicked events in consideration of time efficiency. This empirical study on the sample ratio can provide some justification of the practice. 

\begin{figure} \centering
	\subfigure { \label{fig:lambda_mse}
		\includegraphics[width=0.40\columnwidth]{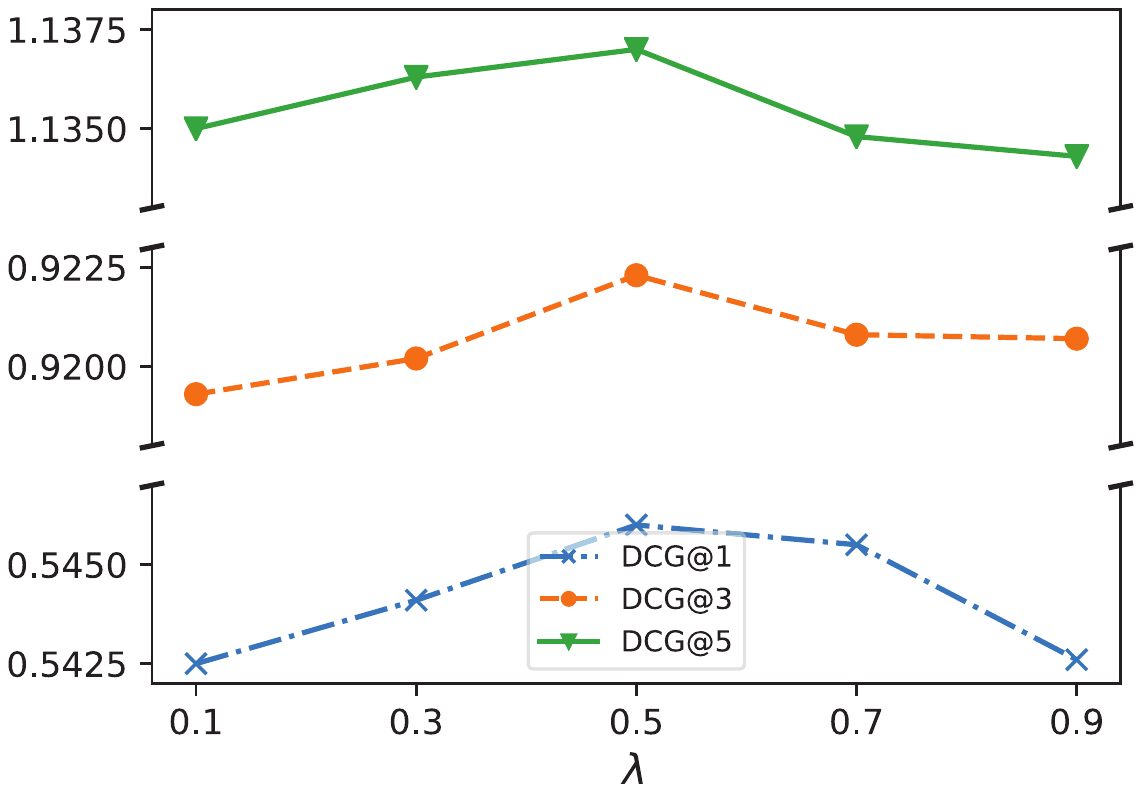}
	}
	\hspace{0.15in}
	\subfigure { \label{fig:sample_ratio}
		\includegraphics[width=0.37\columnwidth]{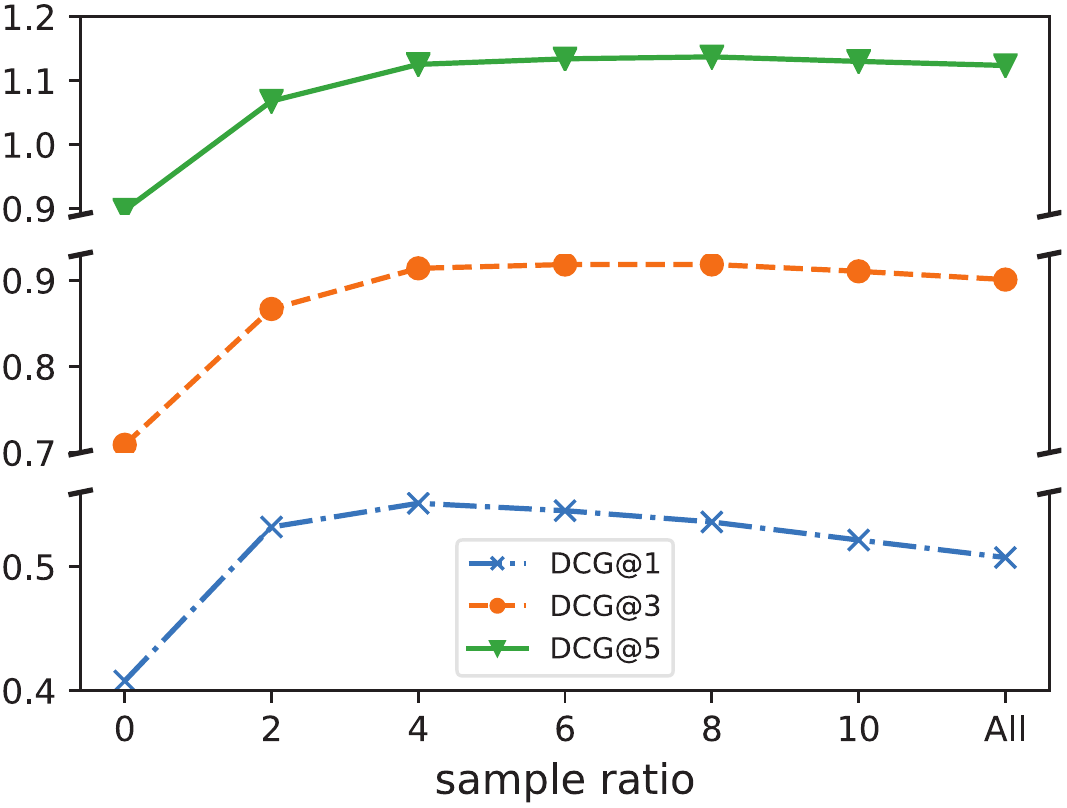}
	}
	\vspace{-0.2in}
	\caption{The effect of the coefficient $\lambda$ for balancing bias and variance, and the sample ratio of unclicked events to clicked events on the ranking performance of DR-MSE.}
	\label{fig:hyperparameter}
	\vspace{-0.3cm}
\end{figure}

\section{Semi-synthetic Data Experiments}


\begin{table*}[t]
\caption{Performance comparison on semi-synthetic datasets based on \textbf{ML-100k}. 
}
\vspace{-0.3cm}
\center
\small
\renewcommand\arraystretch{1.1}
\setlength{\tabcolsep}{5.pt}
\begin{threeparttable}  
\scalebox{0.9}{
\begin{tabular}{c|ccc|ccc}
\hline
Metrics  & \multicolumn{3}{c|}{AUC}                                                                          & \multicolumn{3}{c}{Log-loss}                                                                                       \\
\hline
$\rho$   & 0.5                            & 1                              & 2                              & 0.5                            & 1                              & 2                                         \\
\hline
Naïve   & 0.7250 $\pm$ 0.0001          & 0.6731 $\pm$ 0.0001          & 0.5279 $\pm$ 0.0070          & 0.3178 $\pm$ 0.0000          & 0.3343 $\pm$ 0.0001          & 0.4683 $\pm$ 0.0179  \\
IPS     & 0.7316 $\pm$ 0.0001          & 0.6648 $\pm$ 0.0028          & 0.5263 $\pm$ 0.0055          & 0.3165 $\pm$ 0.0001          & 0.3304 $\pm$ 0.0034          & 0.4789 $\pm$ 0.0132  \\
DR-JL   & 0.7319 $\pm$ 0.0004          & 0.6673 $\pm$ 0.0035          & 0.5703 $\pm$ 0.0032          & 0.3116 $\pm$ 0.0002          & 0.3255 $\pm$ 0.0012          & 0.3607 $\pm$ 0.0014  \\
MRDR & 0.7335 $\pm$ 0.0006          & 0.6765 $\pm$ 0.0021          & 0.5563 $\pm$ 0.0082          & 0.3067 $\pm$ 0.0002          & 0.3238 $\pm$ 0.0006          & 0.3650 $\pm$ 0.0047     \\
\hline
DR-BIAS & \textbf{0.7349 $\pm$ 0.0006*} & \textbf{0.6916 $\pm$ 0.0009*}          & \textbf{0.6073 $\pm$ 0.0054*} & \textbf{0.3064 $\pm$ 0.0001*}          & \textbf{0.3194 $\pm$ 0.0013*}          & \textbf{0.3494 $\pm$ 0.0058}  \\
DR-MSE  & \textbf{0.7359 $\pm$ 0.0002*} & \textbf{0.6928 $\pm$ 0.0020*} & \textbf{0.6084 $\pm$ 0.0168*} & \textbf{0.3059 $\pm$ 0.0001*} & \textbf{0.3193 $\pm$ 0.0028*} & \textbf{0.3477 $\pm$ 0.0084} \\
\hline
\multicolumn{7}{l}{Note: * statistically significant results ($\text{p-value} \leq 0.05$) using the paired-t-test compared with the best baseline.}
\end{tabular}
}   
\end{threeparttable} 
\label{tab:synthetic-result}  
\vspace{-0.25cm}
\end{table*}

In this section, we aim to investigate the robustness of our proposed method through experiments based on semi-synthetic datasets with different levels of selection bias.

\subsection{Experimental Setup}
\subsubsection{Datasets and Preprocessing}\label{data_synthesis}
\textbf{\textbf{MovieLens 100K}\footnote{https://grouplens.org/datasets/movielens/100k/} (\textbf{ML-100K})} is a dataset collected from a movie recommendation service with 100,000 MNAR ratings from 943 users and 1,682 movies. 
We used it to generate semi-synthetic datasets for experiments with the following standard procedures as previous studies~\cite{WSDM_SaitoYNSN20,RecSys_Saito20}.

(1) Obtain an approximation of the true ratings of each user on all items with rating-based matrix factorization~\cite{koren2009matrix}. We denote the predicted rating of a user $u$ on an item $i$ as $\hat{R}_{u,i}$. Then, the ground-truth CVR for conversion generation is generated as follows:
{\small
\[
    p^{cvr}_{u,i} = \sigma(\hat{R}_{u,i}-\epsilon), \forall (u,i)\in\cD,
\]
}
where $\sigma(\cdot)$ is the sigmoid function, and $\epsilon$ controls the level of overall relevance; $\epsilon$ is set to 5 in experiments.

(2) Obtain an approximation of the true observations with logistic matrix factorization~\cite{Johnson2014LogisticMF}. We denote the predicted probability of a user-item pair $(u,i)$ being observed as $\hat{O}_{u,i}$. Then, the ground-truth CTR for generating the click events is defined as follows:
{\small
\[
    p^{ctr}_{u,i} = (\hat{O}_{u,i})^{\rho}, \forall (u,i)\in\cD,
\]
}
where $\rho$ controls the skewness of the distribution of the CTR. A large value of $\rho$ means a huge selection bias in the clicked events and a small number of observed click and conversion events. We set $\rho$ as 0.5, 1, and 2 in the experiments.

(3) Sample binary click and conversion events with Bernoulli sampling based on the ground-truth CTR and CVR as follows:
{\small
\[
    o_{u,i} \sim Bern(p^{ctr}_{u,i}), \; r_{u,i} \sim Bern(p^{cvr}_{u,i}), \; \forall (u,i)\in\cD,
\]
}
where $Bern(\cdot)$ is the Bernoulli distribution. Then, the post-click conversions can be derived as $\{(u,i, r_{u,i})| o_{u,i} = 1\}$.

\subsubsection{Baselines and Implementation}\label{sec-baslines}
The baseline algorithms include the Naive method, IPS~\cite{schnabel2016recommendations}, DR-JL~\cite{Wang-Zhang-Sun-Qi2019}, and MRDR~\cite{MRDR_DL}. The detailed descriptions of the baselines and model implementation are provided in Appendix~\ref{apdx-semi}.

\subsubsection{Evaluation Protocols}
In semi-synthetic datasets, we have the ground-truth user preference information and the level of selection bias of the considered datasets, so that we can investigate model robustness through experiments. We generate the semi-synthetic datasets by setting $\rho$ as 0.5, 1 and 2. The biased set consists of the clicked events generated by the procedure described in Section~\ref{data_synthesis}, which is further divided into a training set (90\%) and a validation set (10\%). We conduct experiments in each setting for 10 times and report the average results. Note that larger value of $\rho$ means higher selection bias and less clicked events for training because of lower propensity. We use AUC and Log-loss on test sets to evaluate the ranking performance and the relevance prediction, respectively. The test set consists of user-item pairs randomly sampled from the unclicked ones, and we uniformly sample 50 items for each user in the experiments. 

\subsection{Results \& Discussion}
Our method DR-MSE has the best AUC scores and Log-loss results across all the considered levels of selection bias ($\rho=0.5, 1, 2$). It demonstrates that DR-MSE can achieve better ranking performance and relevance prediction. DR-BIAS also has impressive performance and outperforms MRDR significantly, which is probably because DR-BIAS achieves smaller bias by magnifying the penalty of the clicked events with low propensity while minifying those with high propensity.
With the increase of the power $\rho$, the performance of IPS drops dramatically, and was even worse than that of the Naive method. It shows that IPS suffers heavily from the high variance issue. Doubly robust learning approaches, including DR-JL, MRDR, DR-BIAS and DR-MSE, have better robustness against the selection bias and demonstrate better results compared with the IPS method. Our proposed DR-MSE performs the best because of its bias and variance reduction characteristics.


\vspace{-0.1cm}
\section{Related Work}
\subsection{Approaches to CVR Estimation}
In practice, CTR prediction models are commonly applied to CVR prediction task due to their inherent similarity. These CTR prediction approaches include logistic regression based methods~\cite{richardson2007predicting,CTR_LR_IJCSIS17}, factorization machine based methods~\cite{rendle2010factorization,juan2016field}, deep learning based methods~\cite{cheng2016wide,wang2017deep,guo2017deepfm,CIKM_WangZDSZHYB19}, etc.
In addition, many approaches are specially designed for CVR prediction because of several unique and critical issues of the task, such as delayed feedback~\cite{KDD_Chapelle14,AAAI_ywjzqdaz21,IJCAI_SuZDZYWBXHY20}, data sparsity~\cite{ESMM18,SIGIR_WenZWLBLY20} and selection bias~\cite{Multi_IPW,MRDR_DL}. 
In this paper, we mainly focus on tackling the selection bias issue.

\textbf{Selection bias} refers to the distribution drift between the train and inference data, which is widely studied recently~\cite{ESMM18,Multi_IPW,RecSys_Saito20,MRDR_DL}. Some existing multi-task learning methods, such as ESMM~\cite{ESMM18} and ESM$^2$~\cite{SIGIR_WenZWLBLY20}, can alleviate the selection bias, but they are heuristic methods and lack theoretic guarantee. 
Further, the author in~\cite{Multi_IPW} tried to use DR method to debias CVR prediction and proposed a model namely Multi\_DR with theoretic guarantee. But they only validated the proposed methods with the biased training and testing sets.
The authors in~\cite{SIGIR_SaitoMY20} proposed a dual learning algorithm for simultaneously tackling the delayed feedback issue and the selection bias issue. MRDR~\cite{MRDR_DL} designs a new loss for the imputation model to reduce the variance of Multi\_DR~\cite{Multi_IPW}. However, it might still suffer from the high bias of DR method due to the incorrect estimations of both propensity scores and imputed errors (which is common in practice). To tackle these problems, in this paper, we proposed a generalized doubly robust learning framework for debiasing CVR prediction, which enables us to propose two new DR methods with more favorable properties.

\vspace{-0.05in}
\subsection{Debiasing in Recommendation Tasks}
Recent years have witnessed many contributions on incorporating the causal inference idea into the recommendation domain for unbiased learning~\cite{schnabel2016recommendations,Wang-Zhang-Sun-Qi2019}.
For example, ~\cite{schnabel2016recommendations} explains the recommendation problem by a treatment-effect model, and designs an IPS based method to remove the bias in the observed data based on explicit feedback.~\cite{Wang-Zhang-Sun-Qi2019} improves over the IPS based method by designing a doubly robust learning approach.
In addition, several existing works~\cite{bonner2018causal,liu2020general,LTD_debias,AutoDebias} design debiasing models by leveraging the available small set of unbiased data.
Though these methods have achieved many successes in debiasing recommendation tasks, none of them are specially proposed for CVR prediction. How to design an unbiased learning algorithm for CVR prediction is highly important and needs to be studied further.

\vspace{-0.05in}
\section{Conclusion}
We have proposed a generic doubly robust (DR) learning framework for debiasing CVR prediction based on the theoretical analysis of the bias, variance and generalization bounds of existing DR methods. This framework enables us to develop a series of new estimators with different desired characteristics to accommodate different application scenarios in CVR prediction. In particular, based on the framework, we proposed two new DR methods, namely DR-BIAS and DR-MSE, which are designed to further reduce the bias and achieve a better bias-variance trade-off. 
In addition, we propose a novel tri-level  optimization for DR-MSE, and the corresponding efficient training algorithm. Finally, we empirically validate the effectiveness of the proposed methods by extensive experiments on both semi-synthetic and real-world datasets.

\section*{Acknowledgements}
This study was partly supported by grants from the National Science and Technology Major Project of the Ministry of Science and Technology of China (No.2021YFF0901400). We appreciate the support from Mindspore\footnote{\url{https://www.mindspore.cn}}, which is a new deep learning computing framework.
\bibliographystyle{ACM-Reference-Format}
\balance
\bibliography{reference}

\clearpage
\appendices

\section{Proof of lemmas} \label{proof_lemma}
This supplementary material contains the proofs of Lemma 1 and Lemma 2.  
For ease of exposition, let $\cL(\hat \bR) = \cL(\hat \bR, \bR^{o})$. 

\setcounter{theorem}{0}
\begin{lemma}[Bias and Variance]
Given 
imputed errors $\hat \bE$ and learned propensities $\hat \bP$ with $\hat p_{u,i} > 0$ for all user-item pairs, the bias and variance of DR estimator are given as 
    {
	\begin{align*}
   Bias[  \cL_{DR}(\hat \bR, \bR^{o}) ] ={}&   \frac{1}{ | \cD | } \Big |  \sum_{(u,i) \in D}  (p_{u,i} - \hat p_{u,i}) \frac{ ( e_{u,i} - \hat e_{u,i})  }{ \hat p_{u,i} }  \Big |, \\
    \V_{\cO}[ \cL_{DR}(\hat \bR, \bR^{o})  ]  ={}&  \frac{1}{ |\cD|^{2} } \sum_{(u,i)\in \cD}    p_{u,i} (1- p_{u,i})  \frac{ ( \hat e_{u,i} - e_{u,i} )^{2} }{  \hat p^{2}_{u,i}  }  . 
    \end{align*}}
\end{lemma}

\begin{proof}  
According to the definition of bias, 
{
	\begin{align*}
		 Bias[  \cL_{DR}(\hat \bR) ] 
		  ={}&  \Big | \bfE_{\cO}[ \cL_{DR}(\hat \bR)]  -   \cL_{ideal}(\hat \bR, \bR)    \Big |    \\
		 ={}&  \Big |  \frac{1}{ |\cD| } \sum_{(u,i) \in \cD} \bfE_{\cO}[  \hat e_{u,i}  +  \frac{ o_{u,i} (e_{u,i} -  \hat e_{u,i}) }{ \hat p_{u, i} }  - e_{u,i} ]      \Big |  		  \\
				={}&    \Big |  \frac{1}{ |\cD| } \sum_{(u,i) \in \cD} [  \hat e_{u,i}  +  \frac{ p_{u,i} (e_{u,i} -  \hat e_{u,i}) }{ \hat p_{u, i} }  - e_{u,i} ]      \Big |  		  \\
		 ={}& \frac{1}{ | \cD | } \Big |   \sum_{(u,i) \in D}  \frac{ p_{u,i} - \hat p_{u,i}  }{ \hat p_{u,i} } ( e_{u,i} - \hat e_{u,i}) \Big |.
  \end{align*}
  }
The variance of $\cL_{DR}(\hat \bR)$ with respect to  click indicator is given as  
{
	\begin{align*}
			  \V_{\cO}[ \cL_{DR}(\hat \bR) ] 
			  ={}&    \frac{1}{ |\cD|^{2} }   \sum_{(u,i) \in \cD}   \V_{\cO}  [   \hat e_{u,i}  +  \frac{ o_{u,i} (e_{u,i} -  \hat e_{u,i}) }{ \hat p_{u, i} } ]     \\
			  ={}&   \frac{1}{ |\cD|^{2} }   \sum_{(u,i) \in \cD}   \V_{\cO}[o_{u,i}] \cdot \left( \frac{e_{u,i} - \hat e_{u,i } }{ \hat p_{u,i} }   \right)^{2}         \\
			  ={}&   \frac{1}{ |\cD|^{2} } \sum_{(u,i)\in \cD}    \frac{ p_{u,i} (1- p_{u,i})  }{    \hat p^{2}_{u,i}  }  ( \hat e_{u,i} - e_{u,i} )^{2}.  
	\end{align*}
	}
\end{proof}

To show the generalization bound of doubly robust estimator, we need the Hoeffding’s inequality for general bounded random variables, which is presented in Lemma 3.

\begin{lemma}[Generalization Bound]
For any finite hypothesis space $\cH$ of prediction matrices,  given   imputed errors $\hat \bE$ and learned propensities $\hat \bP$,  then  with probability $1 - \eta$, 
    {
	\begin{align*} 
	  \cL_{ideal}(\hat \bR^{*}, \bR)  
	  \leq{}&  \cL_{DR}(\hat \bR^{*}, \bR^{o})  +   \underbrace{\frac{1}{ | \cD | }  \sum_{(u,i) \in \cD}   \frac{ | p_{u,i} - \hat p_{u,i} |  }{ \hat p_{u,i} }  |  e_{u,i} - \hat e_{u,i}^{*}  | }_{\text{Bias term}}  \\
	   &  +  \underbrace{ \sqrt{ \frac{ \log(2|\cH | /\eta)  }{ 2 |\cD |^{2}  }   \sum_{(u,i)\in \cD} (  \frac{  e_{u,i} - \hat e_{u,i}^{\dag}  } { \hat p_{u,i} }  )^{2}   } }_{\text{Variance term}},       
	   \end{align*}
	   }
where $\hat e_{u,i}^{\dag}$ is the prediction error corresponding to the prediction matrix $\hat \bR^{\dag} = \arg \max_{ \hat \bR^{h} \in \cH } \sum_{(u,i) \in \cD } (e_{u,i} -\hat e_{u,i}^{h})^2 /\hat p_{u,i}^2$, $\hat e_{u,i}^{*}$ is the prediction error associated with  $\hat \bR^{*}$. 
\end{lemma}

\begin{proof}  We first note that 
{
	\begin{align}
	 &  \cL_{ideal}(\hat \bR^{*}, \bR) -    \cL_{DR}(\hat \bR^{*})  \notag \\
	  ={}&  \cL_{ideal}(\hat \bR^{*}, \bR)  - \bfE_{\cO}[  \cL_{DR}(\hat \bR^{*}) ] + \bfE_{\cO}[  \cL_{DR}(\hat \bR^{*}) ] -  \cL_{DR}(\hat \bR^{*})  \notag \\
	  \leq{}&  Bias[  \cL_{DR}(\hat \bR^{*}) ]  +   \bfE_{\cO}[  \cL_{DR}(\hat \bR^{*}) ] -  \cL_{DR}(\hat \bR^{*})  \notag \\
	  \leq{}&  \frac{1}{ | \cD | }  \sum_{(u,i) \in \cD}   \frac{ | p_{u,i} - \hat p_{u,i} |  }{ \hat p_{u,i} }  |  e_{u,i} - \hat e_{u,i}^{*}  |  +  \bfE_{\cO}[  \cL_{DR}(\hat \bR^{*}) ] -  \cL_{DR}(\hat \bR^{*}) \label{A1}. 
	  	\end{align}
	  	}
Next we focus on analyzing $\bfE_{\cO}[  \cL_{DR}(\hat \bR^{*}) ] -  \cL_{DR}(\hat \bR^{*})$.  
  By Hoeffding's inequality in Lemma 3,  let $X_{u,i} = \frac{ o_{u,i} (e_{u,i} - \hat e_{u,i})  }{ \hat p_{u,i}  }$, then $M_{u,i} - m_{u,i} = \frac{|e_{u,i} - \hat e_{u,i}|  }{ \hat p_{u,i} }$, and for any $\epsilon > 0$, we have  
  {
	\begin{align*}
	 & \P \big \{ \big |   \cL_{DR}(\hat \bR^{*}) - \bfE_{\cO}[  \cL_{DR}(\hat \bR^{*}) ]  \big |  \leq \epsilon   \big  \}  \\
={}&  1 -  \P \big \{  \big |   \cL_{DR}(\hat \bR^{*}) - \bfE_{\cO}[  \cL_{DR}(\hat \bR^{*}) ]  \big |  > \epsilon   \big  \}  \\
\geq{}&  1 -   \P \big \{ \sup_{ \hat \bR^{h} \in \cH } \big |   \cL_{DR}(\hat \bR^{h}) - \bfE_{\cO}[  \cL_{DR}(\hat \bR^{h}) ]  \big |  > \epsilon   \big  \}  \\
\geq{}&  1 -  \sum_{h=1}^{\cH}  \P \big \{ \big |   \cL_{DR}(\hat \bR^{h}) - \bfE_{\cO}[  \cL_{DR}(\hat \bR^{h}) ]  \big |  > \epsilon   \big  \}  \\ 
={}&   {\scriptsize 1-\sum_{h=1}^{\cH}  \P \big \{ \big | \sum_{(u,i)\in \cD } \big( \frac{ o_{u,i} (e_{u,i} - \hat e_{u,i}^{h})  }{ \hat p_{u,i}  }   -  \bfE_{\cO} \big ( \frac{ o_{u,i} (e_{u,i} - \hat e_{u,i}^{h})  }{ \hat p_{u,i}  } \big ) \big)  \big |  > \epsilon  |\cD| \big  \} }	\\
\geq{}&  1-  \sum_{h=1}^{\cH}  2   \exp\big \{  -  2  \epsilon^{2} |\cD|^{2} \big /  \sum_{(u,i)\in \cD} (  \frac{ e_{u,i} - \hat e_{u,i}^{h}   } { \hat p_{u,i} }  )^{2}     \big \}\\
 \geq{}&1 - 2  |\cH|  \exp\big \{  -  2  \epsilon^{2} |\cD|^{2} \big /  \sum_{(u,i)\in \cD} (  \frac{ e_{u,i} - \hat e_{u,i}^{\dag}   } { \hat p_{u,i} }  )^{2}     \big \}.   
	\end{align*}
	}
Letting $2 |\cH| \exp\big \{  -  2  \epsilon^{2} |\cD|^{2} \big /  \sum_{(u,i)\in \cD} (  \frac{ e_{u,i} - \hat e_{u,i}^{\dag}   } { \hat p_{u,i} }  )^{2}     \big \} = \eta$ yields that 
{
	  	\[       \epsilon =  \sqrt{ \frac{ \log(2 |\cH | /\eta)  }{ 2 |\cD |^{2}  }   \sum_{(u,i)\in \cD} (  \frac{  e_{u,i} - \hat e_{u,i}^{\dag}  } { \hat p_{u,i} }  )^{2}   }.         \]
	  	}
 Then with probability $1- \eta$, we have 
 {
 	\begin{equation} \label{A2}      \bfE_{\cO}[  \cL_{DR}(\hat \bR^{*}) ] -  \cL_{DR}(\hat \bR^{*}) \leq   \sqrt{ \frac{ \log(2 |\cH | /\eta)  }{ 2 |\cD |^{2}  }   \sum_{(u,i)\in \cD} (  \frac{ e_{u,i} - \hat e_{u,i}^{\dag}   } { \hat p_{u,i} }  )^{2}   }.     \end{equation}
 	}
Lemma 2 follows immediately from inequalities (\ref{A1}) and  	(\ref{A2}). 
\end{proof}

\begin{lemma}[Hoeffding’s inequality for general bounded random variables] Let $X_{1}, ..., X_{N}$ be independent random variables. Assume that $X_{i} \in [m_{i}, M_{i}]$ for every $i$, Then, for any $\epsilon > 0$, we have   
{
	\[      \P \big \{ \big | \sum_{i=1}^{N} X_{i}  - \sum_{i=1}^{N} \bfE X_{i}  \big |  > \epsilon   \big  \}  \leq 2 \exp\big \{  - \frac{ 2  \epsilon^{2} }{ \sum_{i=1}^{N} (M_{i} - m_{i})^{2} }  \big \}.       \]
	}
\end{lemma} 
\begin{proof}
	The proof can be found in Theorem 2.2.6 of \cite{Vershynin2018}. 
\end{proof}

\section{Experimental Settings on \textbf{Coat}, \textbf{Yahoo}, and Semi-Synthetic Datasets}\label{apdx-semi}
Here, we provide more detailed experimental settings on \textbf{Coat}, \textbf{Yahoo}, and semi-synthetic datasets generated from \textbf{ML-100K}.

\subsection{Datasets}\label{apdx-data}
\begin{itemize}[leftmargin=5.5mm]
	\item {\bf \textbf{Coat Shopping}}: It contains a MNAR training set and a MAR testing set. Specifically, there are 6,960 five-star ratings from 290 Amazon Mechanical Turkers on an inventory of 300 coats in the training set. There are 4,640 ratings collected from the 290 workers on 16 randomly selected coats in the testing set.
	\item {\bf \textbf{Yahoo! R3}}: It includes a MNAR training set with 311,704 five-star ratings from 15,400 users and 1,000 songs, and a MAR testing set with 54,000 ratings from 5,400 users on 10 randomly selected songs.
\end{itemize}		
To make the two datasets consistent with the CVR prediction task, we further preprocess them following previous studies~\cite{RecSys_Saito20,MRDR_DL}:	
 	\begin{enumerate}
	\item  The conversion label $r_{u,i}$ is defined as 1 if the rating of item $i$ by user $u$ is greater than or equal to 4, and 0 otherwise.
	\item   The click indicator $o_{u,i}$ is defined as 1 if user $u$ rated item $i$, and 0 otherwise.
    \item The sets of observed potential conversion labels $r_{u,i}(1)$ is denoted as $\bR^{o} = \{ r_{u,i}(1) \mid o_{u,i} = 1 \} = \{ r_{u,i} \mid o_{u,i} = 1 \}$.  					
	\end{enumerate}
For both datasets, we split the corresponding MNAR dataset into a training (90\%) and a validation (10\%) sets, while all the MAR data is set to testing set. In addition, we restrict our samples to the users with at least one conversion behavior in the testing set as~\cite{RecSys_Saito20,MRDR_DL}. 

\vspace{-0.1cm}

\subsection{Baselines}\label{apdx-semi-baslines}
We compare our proposed methods with the following baselines: 
\begin{itemize}[leftmargin=5.5mm]
\item {\bf Naive}: It directly uses the naive estimator as the loss function for CVR prediction.
\item {\bf IPS}~\cite{schnabel2016recommendations}: It uses the inverse propensity reweighting approach to adjust the distribution of the biased training data.
\item {\bf DR-JL}~\cite{Wang-Zhang-Sun-Qi2019}: It proposes a doubly robust learning model which jointly trains the imputation model and prediction model.
\item {\bf MRDR}~\cite{MRDR_DL}: It is the state-of-the-art model for debiasing CVR prediction, which reduces the variance of doubly robust learning method by designing a new loss for the imputation model.
\end{itemize}
For all considered methods, we follow prior work~\cite{MRDR_DL} to use factorization machine (FM)~\cite{rendle2010factorization} for both CTR and CVR predictions in experiments of \textbf{Coat}, \textbf{Yahoo}, and the semi-synthetic datasets. The CTR prediction model is firstly learned with FM, and used to generate the CTR scores for inverse propensity weighting as~\cite{Wang-Zhang-Sun-Qi2019,MRDR_DL}.

\vspace{-0.1cm}


\subsection{Model Implementation}\label{apdx-model-implement}
We implement all models with TensorFlow~\cite{tensorflow2015-whitepaper} and optimize them with mini-batch Adam~\cite{kingma2014adam}.
We determine the hyper-parameters of each model based on grid search, and the search ranges for the embedding size, batch size, learning rate, L2 regularization coefficient, and sample ratio of unclicked events to clicked events are set as \{16, 32, 64, 128, 256\}, \{256, 512, 1024, 2048\}, \{5e-5, 1e-4, 5e-4, 1e-3, 5e-3, 1e-2\}, \{1e-5, 5e-5, 1e-4, 5e-4, 1e-3, 5e-3\}, and \{2, 4, 6, 8\}, respectively.
The best configuration for each method is determined based on the ranking performance on the validation set.

\section{Experimental Settings on Dataset \textbf{Product}}\label{apdx-product}
\subsection{Baselines}
We further provide some descriptions of the baselines as follows:
\begin{itemize}[leftmargin=5.5mm]
    \item \textbf{DCN}~\cite{wang2017deep}: It is a widely used deep CTR prediction model with a naive estimator. It consists of a deep network and a cross network for feature interaction learning. It is the base model for building all other models.
    \item \textbf{ESMM}~\cite{ESMM18}: It is a multi-task learning model that jointly optimizes CTR prediction and CTCVR prediction.
    \item \textbf{DR-JL}~\cite{Wang-Zhang-Sun-Qi2019}: This model is proposed for debiasing rating prediction by designing a doubly robust learning approach that jointly trains the error imputation model and prediction model. We adapt it for CVR prediction on large-scale dataset with the model architecture shown in Figure~\ref{fig:causalmtl4.1}.
    \item \textbf{Multi\_IPW}~\cite{Multi_IPW}: This model tackles the selection bias in CVR prediction with the inverse propensity weighting approach. It jointly optimizes the CTR loss and IPS based CVR loss.
    \item \textbf{Multi\_DR}~\cite{Multi_IPW}: This model tackles the selection bias in CVR prediction with the doubly robust learning approach inspired by the DR-JL method.
    \item {\bf MRDR}~\cite{MRDR_DL}: It is the state-of-the-art model for debiasing CVR prediction, which reduces the variance of DR method by designing a new loss for the imputation model. However, in the original paper, no experiments on large-scale datasets have been conducted. The original model implementation is not suitable for large-scale dataset, thus we adapt it for experiments on \textbf{Product} with the model architecture shown in Figure~\ref{fig:causalmtl4.1}.
\end{itemize}

For DCN, we train two separate models for CTR and CVR predictions, respectively, and then combine the predictions of these two tasks to obtain the prediction of CTCVR. Besides, the prediction models of DR based methods, including DR-JL, MRDR, DR-BIAS and DR-MSE, are adapted into a multi-task learning framework presented in Figure~\ref{fig:causalmtl4.1} to jointly model CTR prediction and CVR prediction. In other words, the propensity estimation model is jointly learned with the prediction model to handle the data sparsity and selection bias issues. 

\subsection{Model Implementation}
We implement all models with TensorFlow and optimize them with mini-batch Adam. 
For DCN, the embedding size, batch size, learning rate, keep probability of dropout, L2 regularization coefficient and L1 regularization coefficient are set to 150, 8000, 1.5e-4, 0.9, 1e-4, and 1e-8, respectively. The structure of deep network of DCN is set to [1024, 512, 64], and the number of cross layers is set to 3. Other models, including ESMM, DR-JL, Multi\_IPW, Multi\_DR, DR-BIAS and DR-MSE, are built upon DCN. They use similar settings as the baseline DCN for common hyper-parameters. 
Besides, IPS based loss suffers from the high variance issue. We clip the predicted CTR with $\max\{0.03, CTR\}$ to obtain propensity score for both IPS based methods and DR based methods to alleviate this issue. \textbf{Product} contains a training set and a testing set. We report the best results among all training epochs on the testing set of all methods in Table~\ref{tab:huawei-result} for comparison.


\end{document}